\documentclass[Afour,sageh,times]{sagej}

\usepackage{moreverb,url}

\usepackage[colorlinks,bookmarksopen,bookmarksnumbered,citecolor=red,urlcolor=red]{hyperref}

\usepackage{graphicx}%
\usepackage{multirow}%
\usepackage{amsmath,amssymb,amsfonts}%
\usepackage{amsthm}%
\usepackage{mathrsfs}%
\usepackage[title]{appendix}%
\usepackage{xcolor}%
\usepackage{textcomp}%
\usepackage{manyfoot}%
\usepackage{booktabs}%
\usepackage{algorithmicx}%
\usepackage{algpseudocode}%
\usepackage{listings}%
\usepackage{tabularx}
\usepackage{comment}
\usepackage{pgfplots}
\pgfplotsset{compat=1.18}
\usepackage{tikz}

\usepackage[ruled,vlined,linesnumbered]{algorithm2e}
\usepackage{hyperref}
\usepackage{fancyhdr,graphicx,amsmath,amssymb}
\usepackage{makecell}
\usepackage{amsmath}
\usepackage{tikz}
\usepackage{amsthm}

\newtheorem{theorem}{Theorem}%
\newtheorem{proposition}[theorem]{Proposition}%
\newtheorem{property}[theorem]{Property}%

\newcommand\BibTeX{{\rmfamily B\kern-.05em \textsc{i\kern-.025em b}\kern-.08em
T\kern-.1667em\lower.7ex\hbox{E}\kern-.125emX}}

\begin{document}

\title{Forecast Sports Outcomes under Efficient Market Hypothesis: Theoretical and Experimental Analysis of Odds-Only and Generalised Linear Models}

\author{Kaito Goto\affilnum{1}, Naoya Takeishi\affilnum{1} and Takehisa Yairi\affilnum{1}}

\affiliation{\affilnum{1}Research Center for Advanced Science and Technology, The University of Tokyo, Japan}

\corrauth{Kaito Goto, 
Research Center for Advanced Science and Technology,
The University of Tokyo,
153-8904, Komaba 4-6-1, Meguro-ku
Tokyo, Japan}

\email{kaitogoto@g.ecc.u-tokyo.ac.jp}


\maketitle

\section{Abstract}
Converting betting odds into accurate outcome probabilities is a fundamental challenge in order to use betting odds as a benchmark for sports forecasting and market efficiency analysis. In this study, we propose two methods to overcome the limitations of existing conversion methods. Firstly, we propose an odds-only method to convert betting odds to probabilities without using historical data for model fitting. While existing odds-only methods, such as Multiplicative, Shin, and Power exist, they do not adjust for biases or relationships we found in our betting odds dataset, which consists of 90014 football matches across five different bookmakers. To overcome these limitations, our proposed Odds-Only-Equal-Profitability-Confidence (OO-EPC) method aligns with the bookmakers' pricing objectives of having equal confidence in profitability for each outcome. We provide empirical evidence from our betting odds dataset that, for the majority of bookmakers, our proposed OO-EPC method outperforms the existing odds-only methods. Beyond controlled experiments, we applied the OO-EPC method under real-world uncertainty by using it for six iterations of an annual basketball outcome forecasting competition. Secondly, we propose a generalised linear model that utilises historical data for model fitting and then converts betting odds to probabilities. Existing generalised linear models attempt to capture relationships that the Efficient Market Hypothesis already captures. To overcome this shortcoming, our proposed Favourite-Longshot-Bias-Adjusted Generalised Linear Model (FL-GLM) fits just one parameter to capture the favourite-longshot bias, providing a more interpretable alternative. We provide empirical evidence from historical football matches where, for all bookmakers, our proposed FL-GLM outperforms the existing multinomial and logistic generalised linear models.

\section{Introduction}
\label{sec:1}

Professional sport is a growing industry \cite{SportsIndustryOverview}, which increases the availability of betting odds datasets for analysing market efficiencies and biases in betting odds markets to estimate the probability of sports outcomes. However, estimating the probability of sports outcomes is challenging because many factors, such as home advantage, player quality, team chemistry, and others, can influence sports outcomes \cite{LongTermFootballPrediction}.

In sports betting markets, bookmakers set betting odds that are high enough to attract bettors but low enough to be profitable. Rational bettors then attempt to find betting odds that are set above fair odds to generate a long-term profit. Under the Efficient Market Hypothesis (EMH) \cite{EfficientMarketHypothesis}, all relevant information about a sports outcome is assumed to be reflected in the betting odds. Consequently, under EMH, bettors cannot find betting odds that are set above fair odds to secure a long-term profit. This enables economic research on market efficiency \cite{BettingOddsEfficiencyComparison} and portfolio management \cite{BettingStrategyNBA} to be conducted using betting odds.

Sports professionals and researchers often use betting odds under the EMH to account for the many factors that influence sports outcomes. Betting odds have been shown to correlate more strongly with actual outcomes than official rankings. For example, \cite{RatingsBettingOddsComparison} showed that odds-based predictions outperformed FIFA rankings when forecasting the results of the Euro 2008 football tournament. As a result, betting odds are often used as a benchmark to evaluate the accuracy of sports forecasting models \cite{BradleyTerryTennisPrediction}, or to propose models that attempt to capture inefficiencies in betting markets \cite{FootballBothTeamScorePrediction}. In addition, betting odds can be used to calculate Elo ratings \cite{EloRating} to measure the strength of football teams \cite{BettingOddsRatingSystem}. Such odds-derived Elo ratings can help sports organisers develop fairer tournament structures and ranking systems. However, relatively few studies have focused on converting betting odds to probabilities, even though this conversion is important for using betting odds as a benchmark and for accurately estimating the probabilities of sports outcomes.

Four notable methods currently convert betting odds to probabilities without using historical data. The multiplicative conversion \cite{MultiplicativeConversion}, two variants of Shin’s conversion \cite{ShinConversionStrumbeljVariant,ShinConversionKizildemirVariant}, and the power conversion \cite{PowerConversion}. The most common approach is the multiplicative conversion \cite{MultiplicativeConversion}, which normalises the inverse odds to estimate the outcome probabilities. However, this method assumes that bettors have the same expected loss on all outcomes. This assumption contradicts the favourite-longshot bias \cite{FLBias}, which indicates that outcomes with a higher probability of winning tend to have a lower expected loss for bettors. In financial markets, this bias implies that in-the-money options have higher expected returns than out-of-the-money options \cite{FLBiasStockMarkets}.

Shin’s conversion has two variants. A numerical variant \cite{ShinConversionStrumbeljVariant} that assumes insider bettors exist only for the successful outcome, and an analytical variant \cite{ShinConversionKizildemirVariant} that allows insider bettors to exist for any outcome. We provide novel propositions to prove that both variants assume that as the proportion of insider bettors in the market decreases, the betting odds approach fair odds. Under this assumption, a betting market with a smaller sum of inverse odds (booksum), which has odds closer to fair odds, should have less accurate betting odds because fewer insiders participate to lower the betting odds for the successful outcome. However, we did not find a significant correlation between the booksum of a market and the accuracy of its betting odds.

The power method raises all inverse odds to the same constant power \cite{PowerConversion} so that the implied probabilities sum to the number of successful outcomes. This assumes a power law relationship \cite{PowerLawRegression} between inverse odds and probabilities, with no intercept term. However, we found that the intercept is statistically significant, which contradicts this assumption.

We propose a novel odds-only method that assumes the amount of independent bets placed on an outcome is linearly related to the inverse odds. We provide novel propositions that this assumption aligns with the bookmaker’s objective of generating the same profit regardless of the outcome. We found empirical evidence that our proposed Odds-Only-Equal-Profitability-Confidence (OO-EPC) method outperforms all existing odds-only methods. We also provide novel propositions and experimental evidence that other odds-only methods rely on assumptions with no supportive evidence, whereas our method aligns with how bookmakers set the betting odds.

On the other hand, generalised linear models utilise historical betting odds for model fitting and then convert betting odds to probabilities. We propose a novel Favourite-Longshot-Bias-Adjusted Generalised Linear Model (FL-GLM) that fits a power law relationship between the inverse odds and probabilities, but with an adaptive intercept to ensure the probabilities sum to the number of successful outcomes. Our proposed FL-GLM converts betting odds to probabilities more accurately than other existing counterparts, which are multinomial logistic regression \cite{MultinomialLogisticRegression} and ordered logistic regression \cite{OrderedLogisticRegression}. Also, our proposed FL-GLM fits just one parameter, making it significantly more interpretable than its existing counterparts, because the fitted parameter can be used to measure the amount of favourite-longshot bias that exists for the given set of historical betting odds. Furthermore, we prove that multiplicative conversion and power conversion are equivalent to our proposed FL-GLM when the parameter achieves a specific condition; we used this property to test the assumptions made by these existing odds-only methods.

We perform experiments on historical data using betting odds for 90,014 football matches across five different bookmakers from the 2012 to 2024 seasons \cite{FootballData}, which provided a large dataset for evaluation. Using this historical dataset, we evaluate the accuracy of the probabilities computed by the odds-only methods and the generalised linear models, and we also measure the draw bias, where we observed that the favourite-longshot bias is less pronounced for draw outcomes than decisive outcomes.

In addition to the experiments, predictive competitions allow us to examine how methods perform under real-world uncertainty. We have made our proposed OO-EPC method available as an open-source repository, which has been publicly acknowledged by more than 10 gold medal-winning and more than 100 medal-winning solutions. The author has applied it to his solution for the predictive modelling competition, and we provide an informal analysis of his results.

Our two novel methods both convert betting odds to probabilities, but under different data availability situations. The proposed OO-EPC method estimates the outcome probabilities directly from the betting odds, without relying on historical data, whereas the proposed FL-GLM uses historical data for model fitting and then converts betting odds to probabilities.

\section{Related Works}
\label{sec:2}

Existing methods for converting betting odds to probabilities can be divided into two groups: odds-only methods and generalised linear models.

\subsection{Odds-Only Methods}

Odds-only methods convert betting odds to probabilities without using any historical data for model fitting. This group includes the multiplicative method \cite{MultiplicativeConversion}, Shin conversion's numerical and analytical variants \cite{ShinConversionStrumbeljVariant,ShinConversionKizildemirVariant}, and the power method \cite{PowerConversion}. These methods are computationally efficient because they do not require model fitting, but are typically less accurate, since they do not utilise historical data. 

\subsubsection{Multiplicative Conversion}
\label{subsec:2.1}

Multiplicative conversion \cite{MultiplicativeConversion} converts betting odds to probabilities without using any parameters by normalising the inverse odds. But this simplicity relies on a naive assumption that the expected loss for bettors is the same for all outcomes. The multiplicative conversion is an analytical solution that is defined in Algorithm~\ref{alg:1}.

\begin{algorithm}[htbp]
\SetAlgoLined
\textbf{Input:} Number of possible outcomes denoted by $k$; Input vector of betting odds denoted by $\boldsymbol{x} = [x_1, x_2, ..., x_k]$; Number of successful outcomes denoted by $t$
    
\textbf{Output:} Output vector of implied probabilities denoted by $\boldsymbol{\hat{y}} = [\hat{y}_1, \hat{y}_2, ..., \hat{y}_k]$
    
Get the vector of inverse odds denoted by  $\boldsymbol{x}^{-1} = [x_1^{-1}, x_2^{-1}, ..., x_k^{-1}]$.
    
Compute booksum denoted by $s = \frac{1}{t}{\sum_{i=1}^{k} x_i^{-1}}$

Normalise probabilities denoted by $\boldsymbol{\hat{y}} = \frac{\boldsymbol{\boldsymbol{x}^{-1}}}{s}$

\Return{$\boldsymbol{\hat{y}}$}
    
\caption{Multiplicative Conversion}\label{alg:1}
\end{algorithm}

The multiplicative conversion assumes the expected loss for all outcomes are equal and can be expressed as $(1-\frac{t}{s})$ where $t$ and $s$ denote the number of successful outcomes and the booksum, respectively. However, this is a significant drawback because we cannot assume the expected loss for all outcomes to be equal, because the favourite-longshot bias \cite{FLBias} implies that the expected loss for bettors tends to be smaller for outcomes that have a higher probability of winning.

\subsubsection{Shin Conversion}
\label{subsec:2.2}

Shin conversion converts betting odds to probabilities without using any parameters by computing the optimal proportion of insider bettors, denoted by $z_{\alpha}$, for the bookmaker's profit maximisation under the given betting odds. There are two variants of the Shin conversion. The numerical variant of the Shin conversion \cite{ShinConversionStrumbeljVariant} assumes that there exists insider bettors for the successful outcome only and converts betting odds to probabilities numerically. Whereas \cite{ShinConversionKizildemirVariant} is a more efficient variant of the Shin conversion because it converts betting odds to probabilities analytically and assumes that there can exist insider bettors for many outcomes.

Numerical variant of Shin conversion is defined in Algorithm~\ref{alg:2}. Analytical variant of Shin conversion is defined in Algorithm~\ref{alg:3}.

\begin{algorithm}[htbp]
\SetAlgoLined
\textbf{Input:} Number of possible outcomes denoted by $k$; Input vector of betting odds denoted by $\boldsymbol{x} = [x_1, x_2, ..., x_k]$; Convergence threshold denoted by $\delta$ is $\delta = 10^
{-12}$ for our experiments; Proportion of insider bettors for previous and current iteration is denoted by $z_{\alpha}$ and $z_{\beta}$ respectively; Number of successful outcomes denoted by $t$

\textbf{Output:} Output vector of implied probabilities denoted by $\boldsymbol{\hat{y}} = [\hat{y}_1, \hat{y}_2, ..., \hat{y}_k]$

Get the vector of inverse odds denoted by  $\boldsymbol{x}^{-1} = [x_1^{-1}, x_2^{-1}, ..., x_k^{-1}]$

$z_{\alpha} = 0$. Notice we initially assume no bettors have inside information

$z_{\beta} = 1/\delta$

Compute booksum denoted by $s = {\sum_{i=1}^{k} x^{-1}_i}$

\While{$|z_{\alpha} - z_{\beta}| > \delta$}{

    $z_{\alpha} \gets z_{\beta}$
    
    $z_{\beta} \gets \frac{\sum_{i=1}^k \sqrt{z_{\alpha}^2+4(1-z_{\alpha})\frac{x_i^{-2}}{s}} - 2}{n-2}$

}

$\boldsymbol{\hat{y}} = \frac{\sqrt{z_{\beta}^2+4(1-z_{\beta})\frac{\boldsymbol{x}^{-2}}{s}} - z_{\beta}}{2(1-z_{\beta})}$

\Return{$\boldsymbol{\hat{y}}$}

\caption{Numerical variant of Shin Conversion}\label{alg:2}
\end{algorithm}

\begin{algorithm}[htbp]
\SetAlgoLined
\textbf{Input:} Number of possible outcomes denoted by $k$; Input vector of betting odds denoted by $\boldsymbol{x} = [x_1, x_2, ..., x_k]$; Proportion of insider bettors for outcome $i$ denoted by $z_i$; Number of successful outcomes denoted by $t$

\textbf{Output:} Output vector of implied probabilities denoted by $\boldsymbol{\hat{y}} = [\hat{y}_1, \hat{y}_2, ..., \hat{y}_k]$

Get the vector of inverse odds denoted by  $\boldsymbol{x}^{-1} = [x_1^{-1}, x_2^{-1}, ..., x_k^{-1}]$

Compute booksum denoted by $s = \sum_{i=1}^{k} x_i^{-1}$

Compute complement vector denoted by $c_i = x_i^{-1} - (s - x_i^{-1})$

Compute $z_i$ for each outcome denoted by $z_i = \frac{(s - 1)(c_i^2 - s)}{s (c_i^2 - 1)}$

Compute $\hat{y}_i$ for each outcome denoted by $\hat{y}_i = \frac{ \sqrt{z_i^2 + 4(1 - z_i) \frac{x_i^{-2}}{s}} - z_i }{2 (1 - z_i)}$

Compute normaliser denoted by $t^* = \frac{1}{t}{\sum_{i=1}^{k} \hat{y}_i}$

Normalise probabilities denoted by $\boldsymbol{\bar{\hat{y}}} = \frac{\boldsymbol{\hat{y}}}{t^*}$

\Return{$\boldsymbol{\bar{\hat{y}}}$}

\caption{Analytical variant of Shin Conversion}\label{alg:3}
\end{algorithm}

Both variants assume that as the proportion of insiders decreases, the betting odds approach fair odds, as proven by the novel propositions we provide in Proposition \ref{prop:1} and Proposition \ref{prop:2}. Thus, as the proportion of insiders decreases, betting markets with odds that are closer to fair odds should be less accurate because there should be less insiders in that market. However, we found no significant correlation between booksum and the accuracy of the betting odds, because as explained later in Table~\ref{tab:5}, the correlation between booksum and accuracy for all conversion methods has a p-value greater than 0.05, which does not support Shin conversion's assumption.

\begin{proposition}
As the proportion of insiders decreases, the betting odds approach fair odds under the numerical variant of the Shin Conversion, when there is one successful outcome.
\label{prop:1}
\end{proposition}

\begin{proof}
Using the same input and notations as Algorithm~\ref{alg:2}:
\[
x_i^{-1} = \frac{1}{x_i}, \quad \text{and} \quad s = \sum_{i=1}^k x_i^{-1}
\]

The numerical variant of Shin conversion defines the implied probabilities \( \boldsymbol{\hat{y}} \) as:
\[
\boldsymbol{\hat{y}} = \frac{\sqrt{z_{\beta}^2+4(1-z_{\beta})\frac{\boldsymbol{x}^{-2}}{s}} - z_{\beta}}{2(1-z_{\beta})}
\]

We examine the behaviour of $s$ as \( z \to 0 \). Take the limit:
\[
\lim_{z_{\beta} \to 0} \boldsymbol{\hat{y}} = \lim_{z_{\beta} \to 0} \frac{\sqrt{z_{\beta}^2+4(1-z_{\beta})\frac{\boldsymbol{x}^{-2}}{s}} - z_{\beta}}{2(1-z_{\beta})} = \frac{\sqrt{4 \cdot \frac{\boldsymbol{x}^{-2}}{s}}}{2} = \frac{\boldsymbol{x}^{-1}}{\sqrt{s}}
\]

Summing over all \( i \), we get:
\[
\lim_{z \to 0} \sum_{i=1}^k \hat{y}_i = \sum_{i=1}^k \frac{x^{-1}_i}{\sqrt{s}} = \frac{1}{\sqrt{s}} \sum_{i=1}^k x^{-1}_i = \frac{s}{\sqrt{s}} = \sqrt{s}
\]

Since the implied probabilities must sum to 1 when there is one successful outcome:
\[
\sum_{i=1}^k \hat{y}_i = 1 \quad \Rightarrow \quad \sqrt{s} = 1 \quad \Rightarrow \quad s = 1
\]

Thus, as \( z_\beta \to 0 \), we have \( s \to 1 \). In other words, as the proportion of insider bettors approaches zero, betting odds approach fair odds when there is one successful outcome, under the numerical variant of the Shin Conversion.
\end{proof}

\begin{proposition}
    As the proportion of insiders decreases, the betting odds approach fair odds under the analytical variant of the Shin Conversion.
\label{prop:2}
\end{proposition}

\begin{proof}
Using the same input and notations as Algorithm~\ref{alg:3}:
\[
x_i^{-1} = \frac{1}{x_i}, \quad \text{and} \quad s = \sum_{i=1}^k x_i^{-1}
\]

Define the complement term for each outcome as:
\[
c_i = x_i^{-1} - (s - x_i^{-1}) = 2x_i^{-1} - s
\]

The insider trading parameter \( z_i \) is computed as:
\[
z_i = \frac{(s - 1)(c_i^2 - s)}{s (c_i^2 - 1)}
\]

The intermediate implied probabilities \( \hat{y}_i \) are given by:
\[
\hat{y}_i = \frac{ \sqrt{z_i^2 + 4(1 - z_i) \frac{x_i^{-2}}{s}} - z_i }{2 (1 - z_i)}
\]

Then the normalised implied probabilities are:
\[
\bar{\hat{y}}_i = \frac{x_i^{-1}}{t^*}, \quad \text{where} \quad t^* = \frac{1}{t}{\sum_{i=1}^{k} x_i^{-1}}
\]

We examine the behaviour as \( z_i \to 0 \). Take the limit:

\[
\lim_{z_i \to 0} \hat{y}_i = \frac{\sqrt{4 \cdot \frac{x_i^{-2}}{s}}}{2} = \frac{x_i^{-1}}{\sqrt{s}}
\]

Therefore:

\[
\lim_{z_i \to 0} \sum_{i=1}^k \hat{y}_i = \sum_{i=1}^k \frac{x^{-1}_i}{\sqrt{s}} = \frac{s}{\sqrt{s}} = \sqrt{s}
\]

The normalisation factor becomes:

\[
t^* = \frac{1}{t} \sum_{i=1}^k \hat{y}_i \xrightarrow[z_i \to 0]{} \frac{\sqrt{s}}{t}
\]

Thus, the normalised probabilities are:

\[
\bar{\hat{y}}_i = \frac{\hat{y}_i}{t^*} \xrightarrow[z_i \to 0]{} \frac{\frac{x^{-1}_i}{\sqrt{s}}}{\frac{\sqrt{s}}{t}} = \frac{x^{-1}_i}{s} \cdot t
\]

Summing these probabilities:

\[
\sum_{i=1}^k \bar{\hat{y}}_i = \sum_{i=1}^k \left( \frac{x^{-1}_i}{s} \cdot t \right) = \frac{t}{s} \sum_{i=1}^k x^{-1}_i = \frac{t}{s} \cdot s = t
\]

Notice $s = t$ implies the booksum is equal to the number of successful outcomes.

Thus, as \( z_i \to 0 \), we have \( s = t \). In other words, as the proportion of insider bettors approaches zero, betting odds approach fair odds under the analytical variant of Shin Conversion.
\end{proof}

\subsubsection{Power Conversion}
\label{subsec:2.3}

Power conversion \cite{PowerConversion} raises the inverse betting odds of all outcomes to the same constant power of $\beta$ to convert betting odds to probabilities. The power conversion is a numerical solution defined in Algorithm~\ref{alg:4}, which assumes a power-law relationship between the betting odds and probabilities without an intercept term. We found empirical evidence that this is an invalid assumption because, as explained later in Table~\ref{tab:8}, the mean value of the intercept term was below 1 for all bookmakers, which contradicts the power conversion's assumption.

\begin{algorithm}[htbp]
\SetAlgoLined
\textbf{Input:} Number of possible outcomes denoted by $k$; Input vector of betting odds denoted by $\boldsymbol{x} = [x_1, x_2, ..., x_k]$; Number of successful outcomes denoted by $t$; Power constant denoted by $\beta$; Small incremental steps for power constant denoted by $\alpha$ bounded $\alpha > 0$

\textbf{Output:} Output vector of implied probabilities denoted by $\boldsymbol{\hat{y}} = [\hat{y}_1, \hat{y}_2, ..., \hat{y}_k]$

Get the vector of inverse odds denoted by  $\boldsymbol{x}^{-1} = [x_1^{-1}, x_2^{-1}, ..., x_k^{-1}]$. This is computed by $\boldsymbol{x}^{-1} = \frac{1}{\boldsymbol{x}}$

$\beta = 1.0$

\While{$\sum_{i=1}^{k} x_i^{-\beta} > t$}{
    $\beta \gets \beta + \alpha $
}

$\boldsymbol{\hat{y}} = {\boldsymbol{x}}^{-\beta}$

\Return{$\boldsymbol{\hat{y}}$}

\caption{Power Conversion}\label{alg:4}
\end{algorithm}

\subsection{Generalised Linear Models}

Generalised linear models that use historical betting data to fit a model before converting betting odds to probabilities. This group includes multinomial logistic regression \cite{MultinomialLogisticRegression}, ordered logistic regression \cite{OrderedLogisticRegression}, and a power law regression approach \cite{PowerLawRegression}. Generalised linear models typically achieve higher accuracy than odds-only methods, but require useful historical data.


Multinomial logistic regression \cite{MultinomialLogisticRegression} attempts to capture the relationship between each combination of categorical betting odds and categorical outcome. However, under EMH, betting odds has already adjusted for information about other betting odds; because all relevant information about a sports outcome is assumed to be reflected in the betting odds under EMH.

Ordered logistic regression \cite{OrderedLogisticRegression} attempts to capture the ordinal relationship between betting odds. But under EMH, betting odds have already adjusted for the ordinal relationship between betting odds.

Power law regression \cite{PowerLawRegression} models the relationship between the independent variable and dependent variable using the power function. But because the output is continuous, we cannot directly use power law regression to convert betting odds to categorical probabilities.

The final function of the Power Law Regression is $\boldsymbol{\hat{y}} = \exp(\beta_0) X^{-\beta}$ where $X$ denotes the independent variable(s); $\boldsymbol{y}$ denotes the dependent variable; $\beta_0,\beta$ denotes the coefficients; $\boldsymbol{\hat{y}}$ denotes the predicted values.

In summary, none of the existing odds-only methods convert betting odds to probabilities. To address this, we propose an odds-only method that converts betting odds to probabilities in a manner that aligns with how bookmakers set the betting odds. On the other hand, existing generalised linear models use betting odds as raw input data, which results in an attempt to capture relationships already captured by the betting odds under EMH. In contrast, we propose a generalised linear model that attempts to correct the favourite-longshot bias without attempting to capture relationships already captured by the betting odds under EMH.

\section{Proposed Methods}
\label{sec:3}

We propose two novel methods for converting betting odds into probabilities. An odds-only method that converts betting odds to probabilities that does not require any historical data for model fitting, and a generalised linear model that utilises historical data for model fitting and then converts betting odds to probabilities.

\subsection{Our Proposed Odds-Only-Equal-Profitability-Confidence Method (OO-EPC)}
\label{subsec:3.1}

\begin{algorithm}[htbp]
\SetAlgoLined
\textbf{Input:} Number of possible outcomes denoted by $k$; Input vector of betting odds denoted by $\boldsymbol{x} = [x_1, x_2, ..., x_k]$; Number of successful outcomes denoted by $t$

\textbf{Output:} Output vector of implied probabilities denoted by $\boldsymbol{\hat{y}} = [\hat{y}_1, \hat{y}_2, ..., \hat{y}_k]$

Get the vector of inverse odds denoted by  $\boldsymbol{x}^{-1} = [x_1^{-1}, x_2^{-1}, ..., x_k^{-1}]$. This is computed by $\boldsymbol{x}^{-1} = \frac{1}{\boldsymbol{x}}$

Get the vector of standard errors (SE) for inverse odds denoted by $\boldsymbol{\sigma} = [\sigma_1, \sigma_2, ..., \sigma_k]$. This is computed by $\boldsymbol{\sigma} = \sqrt{\frac{\boldsymbol{\boldsymbol{x}^{-1}}(1-\boldsymbol{\boldsymbol{x}^{-1}})}{\boldsymbol{\boldsymbol{x}^{-1}}}}$. We assume the amount of independent bets placed on an outcome has a linear relationship with the inverse odds and thus the denominator of the standard error function is $\boldsymbol{x}^{-1}$.

Get the amount of SE by which all inverse odds should be decreased to reach the implied probabilities to have equal confidence in profitability for each outcome. This is denoted by $z$. This is computed by $z = \frac{(\sum_{i=1}^{k} \boldsymbol{x}^{-1}_i) - t}{\sum_{i=1}^{k} \sigma_i}$

\If{$\forall i$ $z < \frac{\boldsymbol{x}^{-1}_{i}}{\sigma_{i}}$}{
    Get the output vector of implied probabilities. This is computed by $\boldsymbol{\hat{y}} = \boldsymbol{\boldsymbol{x}^{-1}} - z\boldsymbol{\sigma}$
}
\Else{
    Run Algorithm~\ref{alg:1} from scratch
}

\Return{$\boldsymbol{\hat{y}}$}

\caption{Odds-Only-Equal-Profitability-Confidence Method}\label{alg:8}
\end{algorithm}

We propose a novel odds-only method that reduces all inverse odds by the same amount of standard errors to convert betting odds to probabilities, so bookmakers have the same level of confidence for each outcome that the odds cannot provide a long-term profit for bettors. Specifically, we compute $z$, which is the amount of standard errors each outcome’s inverse betting odds should be reduced by such that the sum of the probabilities of all outcomes equals the number of successful outcomes. Thus, our proposed OO-EPC method analytically converts betting odds to probabilities, as defined in Algorithm~\ref{alg:8}.

The goal of bookmakers, according to \cite{sumpter2016soccermatics}, is to generate the same profit regardless of the outcome. As proven by Proposition \ref{prop:xxx}, this goal of bookmakers implies that the amount of independent bets placed on an outcome has a linear relationship with inverse odds. As a result of this implication, the denominator of the standard error function is set to the inverse odds as shown on step 4 of Algorithm \ref{alg:8}.

\begin{proposition}
    The same profit for the bookmaker regardless of the outcome implies a linear relationship between the amount of independent bets placed.
\label{prop:xxx}
\end{proposition}

\begin{proof}
Using the same input and notations as Algorithm \ref{alg:8}:
\[
x_i^{-1} = \frac{1}{x_i}
\]
Let the vector of weights that is proportional to the amount of independent bets placed be denoted as $\boldsymbol{b} = [b_1, b_2, ..., b_k]$.

Notice that the profit for the bookmaker given outcome $j$ is successful can be computed by:
\[
1 - \frac{b_jx_j}{\sum_{i=1}^{k}{b_i}}
\]
Thus, when bookmakers generate the same profit for themselves regardless of the outcome:
\[
1 - \frac{b_1x_1}{\sum_{i=1}^{k}{b_i}} = 1 - \frac{b_2x_2}{\sum_{i=1}^{k}{b_i}} = ... = 1 - \frac{b_kx_k}{\sum_{i=1}^{k}{b_i}}
\]
Simplifies to:
\[
b_1x_1 = b_2x_2 = ... = b_kx_k
\]
Thus, when $b_i = x^{-1}_i, \ \forall i$ we get:
\[
x_1^{-1}x_1 = x^{-1}_2x_2 = ... = x^{-1}_kx_k = 1
\]
Therefore, when the bookmakers generate the same profit for themselves regardless of the outcome, the amount of independent bets placed on each outcome is proportional to the inverse odds of each outcome.

\end{proof}

Step 5 of Algorithm \ref{alg:8} intuitively means that the bookmakers have equal confidence in the profitability of each outcome's betting odds, because the inverse odds of all outcomes, $\boldsymbol{x}^{-1}$, are decreased by the same amount of each outcome's standard errors, $z\boldsymbol{\sigma}$.

The constraint on step 6 of Algorithm \ref{alg:8} intuitively requires that betting markets have a booksum slightly greater than 1. If this constraint is not satisfied in a given market, our proposed OO-EPC method will output negative probabilities, so in our experiments, we used the multiplicative method when the constraint was dissatisfied. In our historical dataset of betting odds and outcomes for 90,014 football matches across five different bookmakers from the 2012 to 2024 seasons, the constraint was satisfied for $>99.9\%$ of the time, so the proposed OO-EPC method almost never becomes equivalent to the multiplicative conversion.

\subsection{Our Proposed Favourite-Longshot-Bias-Adjusted Generalised Linear Model (FL-GLM)}
\label{subsec:3.2}

We also propose a novel generalised linear model that utilises historical data. As proven by Proposition \ref{prop:5}, the proposed FL-GLM can be viewed as a variant of the power law regression model in which the intercept term is adaptive to ensure that the output probabilities sum to the number of successful outcomes. Our proposed FL-GLM firstly raises the inverse betting odds of all outcomes to the same constant power, in a similar way to the power conversion. Secondly, it performs the multiplicative conversion to ensure the output probabilities sum to the number of successful outcomes, as defined in Algorithm~\ref{alg:9}.

\begin{algorithm}[htbp]
\caption{Favourite-Longshot-Bias-Adjusted Generalised Linear Model}\label{alg:9}
\KwIn{$X^{-1}$: column vectors of inverse betting odds for each category; $Y$: matrix of labels where each row is a one-hot vector for the outcome of the given game; $\alpha$: learning rate}
\KwOut{$\beta$: coefficient; $\bar{\hat{Y}}$: matrix of estimated probabilities}

\Repeat{convergence}{
    Compute the temporary estimated probabilities: $\hat{Y} = X^{-\beta}$;

    Normalise the temporary estimated probabilities for each row: \[
    \bar{\hat{Y}}_{ij} = \frac{\hat{Y}_{ij}}{\sum_{j=1}^{k} \hat{Y}_{ij}}
    \quad 
    \] 
    \[\text{for } i = 1, 2, \dots, n,\; j = 1, 2, \dots, k
    \]
    
    Compute log-likelihood: 
    \[
    \ln L = \sum_{i=1}^{n} \sum_{j=1}^{k} Y_{ij} \ln \bar{\hat{Y}}_{ij}
    \]
    
    Compute gradient $\nabla_{\beta} L$\;
    Update coefficients: $\beta \leftarrow \beta + \alpha \nabla_{\beta} L$\;
}

\Return{$\beta$, $\bar{\hat{Y}}$}

\end{algorithm}

\begin{proposition}
    Our proposed FL-GLM is equivalent to power law regression when $\exp(\beta_0) = \frac{1}{\sum_{j=1}^{k} X_{ij}^{-\beta}}$.
\label{prop:5}
\end{proposition}

\begin{proof}
Using the same input and notations as Algorithm \ref{alg:9}, when the intercept term of the power law regression, denoted by $\exp(\beta_0)$, is replaced by an intercept term that normalises the categorical probabilities so that they sum to the number of successful outcomes, denoted by $\frac{1}{\sum_{j=1}^{k} X_{ij}^{-\beta}}$, we get:

\[\exp(\beta_0) = \frac{1}{\sum_{j=1}^{k} X_{ij}^{-\beta}}\]

When $\exp(\beta_0) = \frac{1}{\sum_{j=1}^{k} X_{ij}^{-\beta}}$, the final function of our proposed FL-GLM can be expressed as:

\[\bar{\hat{Y}}_{i} = \frac{\hat{Y}_{i}}{\sum_{j=1}^{k} \hat{Y}_{ij}} = \frac{X_{i}^{-\beta}}{\sum_{j=1}^{k} X_{ij}^{-\beta}} = \exp(\beta_0){X_{i}^{-\beta}}\] 

Notice that $\exp(\beta_0){X_{i}^{-\beta}}$ is equivalent to the final function of power law regression.
\end{proof}

Thus, Proposition \ref{prop:5} proves that our proposed FL-GLM is an extension of the existing power law regression. The intercept term of our proposed FL-GLM adaptively normalises the categorical probabilities so that they sum to the number of successful outcomes, whereas the intercept term of the existing power law regression is an explicit parameter.

Notice our proposed FL-GLM has only one parameter, whereas the multinomial and ordered logistic generalised linear models have significantly more parameters. This makes our proposed FL-GLM significantly more interpretable. And the single parameter of our proposed FL-GLM, the fitted power constant directly measure the amount of favourite-longshot bias for the given dataset. Because when $\beta = 1$, it is equivalent to the multiplicative conversion as proven by Property \ref{prop:3}, which assumes the favourite-longshot bias does not exist; when $\beta > 1$, it disproportionately reduces the probability of longshots in an attempt to adjust for the favourite-longshot bias. This implies the greater the $\beta$, the greater the favourite-longshot bias in the historical betting odds used for fitting. Also, our proposed FL-GLM becomes equivalent to the power conversion when the normalisation term is 1, as proven by Property \ref{prop:4}.

\begin{property}
    Our proposed FL-GLM is equivalent to multiplicative conversion when $\beta = 1.0$.
\label{prop:3}
\end{property}

\begin{proof} 
Using the same input and notations as step 3 Algorithm \ref{alg:9}, when $X_{i} = \boldsymbol{x} \text{ and } \beta = 1.0$, the final function of our proposed FL-GLM can be expressed as: \[\bar{\hat{Y}}_{i} = \frac{\hat{Y}_{i}}{\sum_{j=1}^{k} \hat{Y}_{ij}} = \frac{X^{-\beta}_{i}}{\sum_{j=1}^{k} X^{-\beta}_{ij}} = \frac{X^{-1}_{i}}{\sum_{j=1}^{k} X^{-1}_{ij}} = \frac{\boldsymbol{x^{-1}}}{\sum_{j=1}^{k} x^{-1}_{j}} = \boldsymbol{\hat{y}}\] 

Notice that $\frac{\boldsymbol{x^{-1}}}{\sum_{j=1}^{k} x^{-1}_{j}} = \boldsymbol{\hat{y}}$ is the final function of multiplicative conversion (Algorithm \ref{alg:1}).
\end{proof}

Thus, Property \ref{prop:3} proves that the multiplicative conversion assumes the power term of our proposed FL-GLM can be assumed to be 1.

\begin{property}
    Our proposed FL-GLM is equivalent to power conversion when $\beta$ is replaced by a vector of parameters that results in the normalisation term equal to 1.
\label{prop:4}
\end{property}

\begin{proof}
Using the same input and notations as Algorithm \ref{alg:9}, notice that for the final function of our proposed FL-GLM, when the normalisation term is removed, and $\beta$ is replaced by a vector of parameters $\beta$, we get:

\[X_i^{-\beta_i} = \bar{\hat{Y_i}}\]

Thus, when $X_{i} = \boldsymbol{x}$ and $X^{-\beta} = \bar{\hat{Y}}$, the final function of our proposed FL-GLM can be expressed as: \[\bar{\hat{Y_i}} = X_i^{-\beta_i} = \boldsymbol{x}^{-\beta} = \boldsymbol{\hat{y}}\] 

Notice $\boldsymbol{x}^{-\beta} = \boldsymbol{\hat{y}}$ is the final function of power conversion (Algorithm \ref{alg:4}).
\end{proof}

Thus, Property \ref{prop:4} proves that the power conversion assumes the normaliser term of our proposed FL-GLM can be assumed to be 1.

In summary, our proposed OO-EPC method is theoretically-driven to ensure each outcome provides the same confidence in profitability for the bookmaker;  whereas our proposed FL-GLM is data-driven to adjust for the favourite-longshot bias measured from historical betting odds data when it is available.

\section{Experiments}
\label{sec:4}

Experiments were conducted for all matches during the 2012 to 2024 football seasons that were available on \url{football-data.co.uk} \cite{FootballData}, which covered 90,014 football matches across five different bookmakers. The bookmakers were Bet365, Bet\&Win, Interwetten, Pinnacle, and William Hill.

The log-loss function is defined as:
\newline
\newline
$L(y, \hat{y}) = - \frac{1}{N} \sum_{i=1}^{N} \sum_{j=1}^{3} y_{ij} \ln(\hat{y}_{ij})$
\newline
\newline
where $N$ denotes the total number of matches in our dataset;
$j$ denotes the match outcome (home win, draw or away win);
$y_{ij}$ is an indicator variable that is 1 if the actual outcome of the $i$-th match is class $j$, and 0 otherwise;
$\hat{y}_{ij}$ denotes the predicted probability that the $i$-th match belongs to class $j$;
$\ln$ denotes the calculated natural-log.
\newline
\newline
We evaluate the estimated probabilities under the log-loss function for two main reasons. Firstly, because it is a proper scoring function for accuracy evaluation, where the function is minimised when the estimated probabilities are the true probabilities. Secondly, it captures the profitability of the estimated probabilities for bookmakers as proven by Proposition \ref{prop:6}. Thus, it allows us to ignore the difference between accuracy and profitability \cite{WhyLogLoss} because both are measured under the log-loss.

\begin{proposition}
    Log-loss captures bookmakers' profitability because it is derived from the payout ratio.
\label{prop:6}
\end{proposition}

\begin{proof}
For a given winning outcome, let $x_A$ and $x_B$ denote the betting odds for bookmakers $A$ and $B$, respectively. Then, the payout ratio of bookmaker $A$ relative to bookmaker $B$ is expressed as:
\[u = x_B/x_A.\]

Thus, payout per winning unit for bookmaker $B$ is $u$-times that of bookmaker $A$, relative to bookmaker $B$ is $u$.

The difference in log-loss between bookmaker $A$ and $B$ for the given outcome can be expressed as: \[v = -\ln(x^{-1}_B) + \ln(x^{-1}_A).\]

For the given outcome, the difference in log-loss between bookmaker $A$ and $B$ can be used to derive the payout ratio of bookmaker $A$ relative to bookmaker $B$ because we can derive $u$ from $v$ by: \[u = \exp(-v).\]

Thus, the difference in mean log-loss between bookmaker $A$ and $B$ can be used to derive the mean payout ratio between bookmaker $A$ and $B$ under fair odds.
\end{proof}

The remainder is organised as follows: Section \ref{subsubsec:4.1} evaluates the accuracy of our proposed methods. We evaluate the estimated probabilities implied by our proposed OO-EPC method in comparison to those implied by existing counterparts, as well as those implied by our proposed FL-GLM in comparison to its existing counterparts. Section \ref{subsubsec:4.2} tests the assumptions made by existing odds-only methods. We test the Shin conversion's assumption that bookmakers with smaller expected losses imply less accurate betting odds. We also test the multiplicative and power conversions' assumptions by testing whether our proposed FL-GLM’s power constant and normaliser terms can be assumed to be 1, respectively. Section \ref{subsubsec:4.4} test the draw bias noticed by \cite{ShinConversionStrumbeljVariant} for odds-only methods and generalised linear models.

\subsection{Experiment for Accuracy of Methods}
\label{subsubsec:4.1}

Table~\ref{tab:joined} provides the mean log-loss under natural-log of the converted probabilities for each odds-only method across the five bookmakers, to compare the accuracy of our proposed OO-EPC method against the existing odds-only methods. A two-tailed bootstrap significance test at the 0.05 level was used to identify significant differences.

Our proposed OO-EPC method had a significantly superior log-loss than all other odds-only methods for the majority of the five bookmakers. Numerical variant of Shin and Power conversions significantly outperformed our proposed OO-EPC method for the Bet365 bookmaker. Power conversion significantly outperformed our proposed odds-only for the William Hill bookmaker as well.

\begin{table*}[htbp]
\centering
\resizebox{\textwidth}{!}{%
\begin{tabular}{|l|c|c|c|c|c|}
\hline
\textbf{Odds-Only Method} & 
\makecell{\textbf{Bet365}} & 
\makecell{\textbf{Bet\&Win}} &
\makecell{\textbf{Interwetten}} & 
\makecell{\textbf{Pinnacle}} &
\makecell{\textbf{William Hill}} \\
\hline
Multiplicative & 1.00372 & 1.00417 & 1.00575 & 1.00449 & 1.00425 \\
Numerical Shin & \textbf{1.00336} & 1.00356 & 1.00477 & 1.00432 & 1.00360 \\
Analytical Shin & 1.00362 & 1.00368 & 1.00479 & 1.00430 & 1.00404 \\
Power & \textbf{1.00334} & 1.00350 & 1.00455 & 1.00431 & \textbf{1.00349} \\
Our Proposed & 1.00341 & \textbf{1.00349} & \textbf{1.00449} & \textbf{1.00428} & 1.00359 \\
\hline
\end{tabular}%
}
\caption{For each \textbf{Odds-Only Method}, the mean log-loss under natural-log of the converted probabilities was computed for each of the five bookmakers. Log-loss scores in bold indicate when our proposed OO-EPC method significantly outperformed other odds-only methods or when other odds-only method(s) significantly outperformed our proposed OO-EPC method. We used a two-tailed bootstrap significance test at a 0.05 threshold.}
\label{tab:joined}
\end{table*}

Table~\ref{tab:joined_bootstrap_regression} provides the mean log-loss under natural-log of the estimated probabilities for each generalised linear model across the five bookmakers, to compare the accuracy of our proposed FL-GLM against the existing generalised linear models. A two-tailed bootstrap significance test at the 0.05 level was used to identify significant differences.

Our proposed FL-GLM had a significantly superior log-loss than all other generalised linear models for all five bookmakers. This outperformance suggests that our regression approach, adjusts for the favourite-longshot bias more effectively than the existing generalised linear models.

\begin{table*}[htbp]
\centering
\resizebox{\textwidth}{!}{%
\begin{tabular}{|l|c|c|c|c|c|}
\hline
\textbf{Generalised Linear Model} & 
\makecell{\textbf{Bet365}} & 
\makecell{\textbf{Bet\&Win}} &
\makecell{\textbf{Interwetten}} & 
\makecell{\textbf{Pinnacle}} &
\makecell{\textbf{William Hill}} \\
\hline
Multinomial Logistic & 1.00573 & 1.00601 & 1.00550 & 1.00557 & 1.00473 \\
Ordered Logistic & 1.00722 & 1.00766 & 1.00684 & 1.00680 & 1.00637 \\
Our Proposed & \textbf{1.00375} & \textbf{1.00404} & \textbf{1.00333} & \textbf{1.00306} & \textbf{1.00239} \\
\hline
\end{tabular} 
}
\caption{For each \textbf{generalised linear model}, the mean log-loss under natural-log of the estimated probabilities was computed for each of the five bookmakers. Log-loss scores in bold indicate when our proposed FL-GLM significantly outperformed all other generalised linear models. We used a two-tailed bootstrap significance test at a 0.05 threshold.}
\label{tab:joined_bootstrap_regression}
\end{table*}

\subsection{Experiment for Assumptions of Existing Odds-Only Methods}
\label{subsubsec:4.2}

Table~\ref{tab:5} provides the correlation between accuracy of converted probabilities and the bookmakers’ margins under each odds-only method. For each combination of bookmaker and season, we calculated the correlation between mean log-loss under natural-log of the converted probabilities and the average booksum. The table provides the correlation coefficient between log-loss and booksum for each odds-only method, along with the corresponding p-value for significance.

\begin{table}[htbp]
    \centering
    \resizebox{\columnwidth}{!}{%
    \begin{tabular}{|l|c|c|}
    \hline
    \textbf{Odds-Only Method} & \makecell{\textbf{Correlation between} \\ \textbf{Log-Loss and Booksum}} & \textbf{p-value} \\
    \hline
    Multiplicative & 0.1085 & 0.4093 \\
    Numerical variant of Shin & 0.0744 & 0.5723 \\
    Analytical variant of Shin & 0.0828 & 0.5294 \\
    Power & 0.0672 & 0.6097 \\
    Our Proposed & 0.0661 & 0.6159 \\
    \hline
    \end{tabular}%
    }
    \caption{Under each \textbf{Odds-Only Method}, the mean log-loss under natural-log of the converted probabilities and the mean booksum were computed for each combination of bookmaker and season. The correlation between the log-loss and booksum is displayed in the \textbf{Correlation between Log-Loss and booksum} column. The significance of the correlation is displayed in the \textbf{p-value} column.}
    \label{tab:5}
\end{table}

All odds-only methods have small positive correlation coefficients, with p-values significantly greater than 0.05, and thus none are statistically significant. These results imply that we have no evidence to support Shin conversion's assumption that markets with fewer insider bettors, result in smaller booksums, have less accurate betting odds. In fact, because the correlation coefficients are all small positives, we have insignificant evidence against Shin conversion's assumption.

Table~\ref{tab:8} provides the key parameters of our proposed favourite-longshot-bias-adjusted regression model for each bookmaker. It shows the fitted power constant $\beta$ for each bookmaker and the mean value of the normalisation term, which is the average of $\sum_{j} \hat{Y}_{ij}$, for that bookmaker. We use Property \ref{prop:3} and Property \ref{prop:4} of our proposed favourite-longshot-bias-adjusted regression to test the assumptions made by the multiplicative and power conversions; respectively.

\begin{table}[htbp]
    \centering
    \resizebox{\columnwidth}{!}{%
    \begin{tabular}{|l|c|c|}
    \hline
    \textbf{Bookmaker} & 
    \makecell{\textbf{Output $\beta$ of} \\ \textbf{FL-GLM}} &
    \makecell{\textbf{Mean $\sum_{j=1}^{k} \hat{Y}_{ij}$ of} \\ \textbf{FL-GLM}} \\
    \hline
    Bet365 & 1.08 & 0.98 \\
    Bet\&Win   & 1.12 & 0.96 \\
    Interwetten   & 1.15 & 0.95 \\
    Pinnacle   & 1.06 & 0.97 \\
    William Hill & 1.12 & 0.96 \\
    \hline
    \end{tabular}%
    }
    \caption{For each \textbf{Bookmaker}, the $\beta$ value of our fitted FL-GLM is displayed on the \textbf{Optimal $\beta$ of FL-GLM} column, and the mean value of the final normaliser term of our fitted proposed FL-GLM is displayed on the \textbf{Mean $\sum_{j=1}^{k} \hat{Y}_{ij}$ of FL-GLM} column.}
    \label{tab:8}
\end{table}

The power constant $\beta$ is above 1.0 for every bookmaker, ranging from 1.06 to 1.15. The normalisation term is below 1.0 in all cases, ranging from 0.95 to 0.98. These results imply limitations of the multiplicative and power conversion methods in adjusting for the biases in betting odds. Because the multiplicative conversion assumes $\beta = 1$, and the power conversion assumes the normaliser equals 1. In contrast, our findings confirm the existence of the favourite-longshot bias, since $\beta > 1$ for all five bookmakers, and the significance of an intercept term for the power-law relationship between betting odds and probabilities, since the mean value of the normalisation term is less than 1 for all five bookmakers.

\subsection{Experiment for Draw Bias in Methods}
\label{subsubsec:4.4}

Figure~\ref{fig:draws_chart} illustrates the expected drawn matches implied by each odds-only method and the actual number of drawn matches, for each of the five bookmakers' datasets of historical betting odds.

\definecolor{myblue}{RGB}{70, 130, 180}
\definecolor{mygreen}{RGB}{46, 139, 87}
\definecolor{myorange}{RGB}{255, 127, 80}
\definecolor{myred}{RGB}{190, 30, 45}
\definecolor{mycyan}{RGB}{0, 150, 165}
\definecolor{mymagenta}{RGB}{150, 75, 130}
\definecolor{mydarkyellow}{RGB}{204, 153, 0}

\begin{figure*}[htbp]
\centering
\begin{tikzpicture}
\begin{axis}[
    ybar,
    area legend,
    enlarge x limits=0.15,
    legend style={at={(0.5,-0.2)}, anchor=north, legend columns=2},
    ylabel={Number of Draws},
    symbolic x coords={Bet365, Bet\&Win, Interwetten, Pinnacle, William Hill},
    xtick=data,
    scaled y ticks=false,
    yticklabel style={/pgf/number format/fixed, /pgf/number format/1000 sep={}},
    nodes near coords,
    every node near coord/.append style={rotate=90, anchor=west, font=\small, text=black, /pgf/number format/1000 sep={}},
    ymin=21000, ymax=24500, 
    width=\textwidth,
    height=8cm,
    bar width=7pt,
    title={Comparison of Expected Number of Draws implied by Each Odds-Only Method and \\ Actual Number of Draws under each Bookmaker},
    title style={align=center},
    ymajorgrids=true,
    grid style=dashed,
]

\addplot[fill=mygreen, draw=black!80] coordinates {(Bet365,23575.9) (Bet\&Win,23583.6) (Interwetten,22812.8) (Pinnacle,23373.6) (William Hill,23583.8)};

\addplot[fill=mycyan, draw=black!100] coordinates {(Bet365,23299.6) (Bet\&Win,23242.2) (Interwetten,22450.6) (Pinnacle,23226.6) (William Hill,23266.9)};

\addplot[fill=mymagenta, draw=black!100] coordinates {(Bet365,23006.8) (Bet\&Win,22871.7) (Interwetten,22051.6) (Pinnacle,23075.0) (William Hill,22925.9)};

\addplot[fill=mydarkyellow, draw=black!100] coordinates {(Bet365,23119.0) (Bet\&Win,23029.8) (Interwetten,22221.4) (Pinnacle,23133.6) (William Hill,23054.0)};

\addplot[fill=black!50, draw=black!100] coordinates {(Bet365,23095.3) (Bet\&Win,22985.6) (Interwetten,22176.2) (Pinnacle,23120.1) (William Hill,23033.7)};

\addplot[fill=black!100] coordinates {(Bet365,23649) (Bet\&Win,23534) (Interwetten,22528) (Pinnacle,23414) (William Hill,23224)};

\legend{Multiplicative, Numerical variant of Shin, Analytical variant of Shin, Power, Ours (OO-EPC), Actual Draws}
\end{axis}
\end{tikzpicture}
\caption{Comparison of expected \textbf{Number of Draws} implied by each \textbf{Odds-Only Method} and the Actual Draws under each bookmaker. Each given group of six bars represent an evaluation under the given bookmaker; each given colour represents the amount of expected draws implied by the given odds-only method or the actual number of draws.}
\label{fig:draws_chart}
\end{figure*}
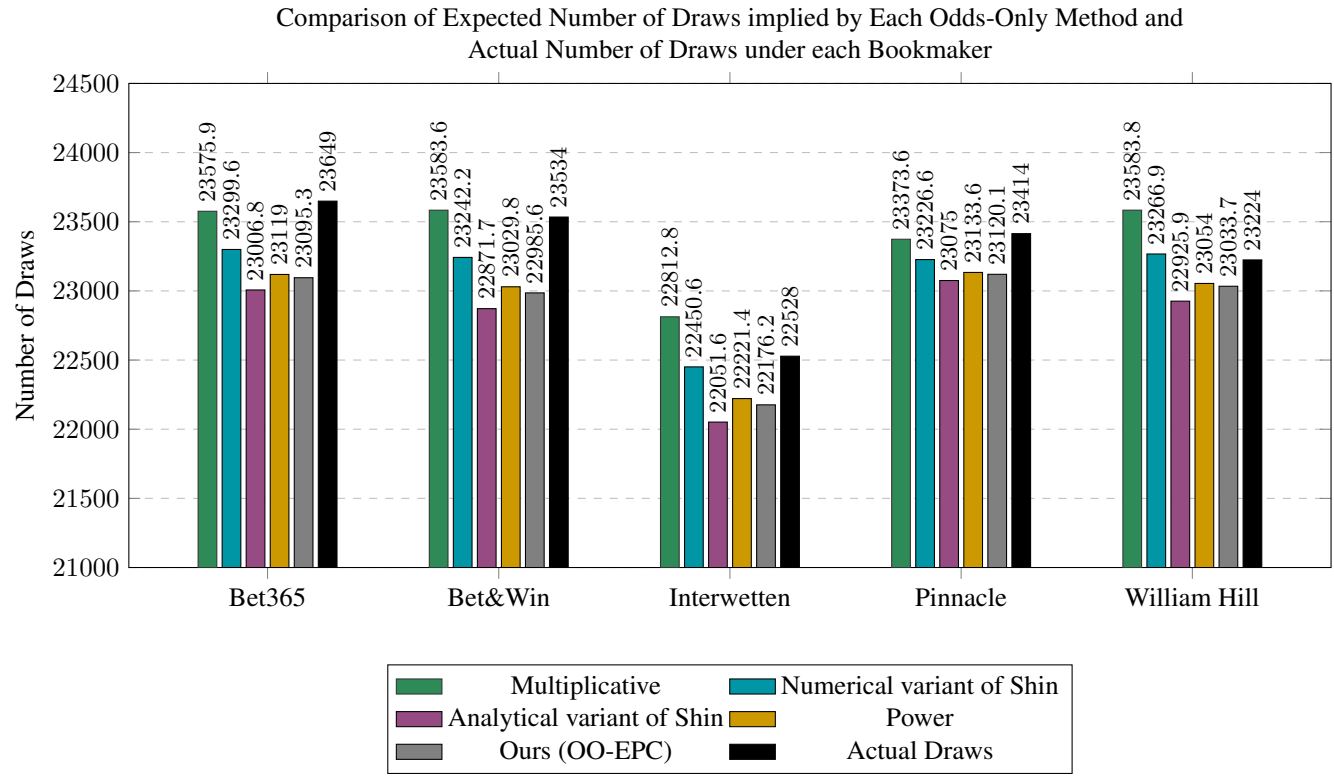

All the odds-only methods that attempt to adjust for the favourite-longshot bias, which are numerical Shin, analytical Shin, power, and our proposed method, tend to underestimate the number of draws. 

Based on a two-tailed Poisson significance test at the 0.05 level, we found that for Bet 365, numerical Shin, analytical Shin, power, and our proposed method implied significantly fewer number of draws than the actual number. For Bet\&Win and Interwetten, it was analytical Shin, power, and our proposed method. For Pinnacle and William Hill, it was analytical Shin.

In contrast, the expected number of draws implied by the multiplicative method, which does not adjust for the favourite–longshot bias, does not significantly differ from the actual number of draws observed for all five bookmakers.

We hypothesise that this implies that draw outcomes have a weaker favourite–longshot bias than decisive outcomes, where methods that attempt to adjust for the favourite-longshot bias underestimate draw probabilities because draw outcomes tend to be longshots, and the favourite-longshot bias for draw outcomes is overestimated by those methods.

Figure~\ref{fig:regression_draws_chart} illustrates the expected drawn matches implied by each generalised linear model and the actual number of drawn matches in each of the five bookmakers' datasets of historical betting odds.

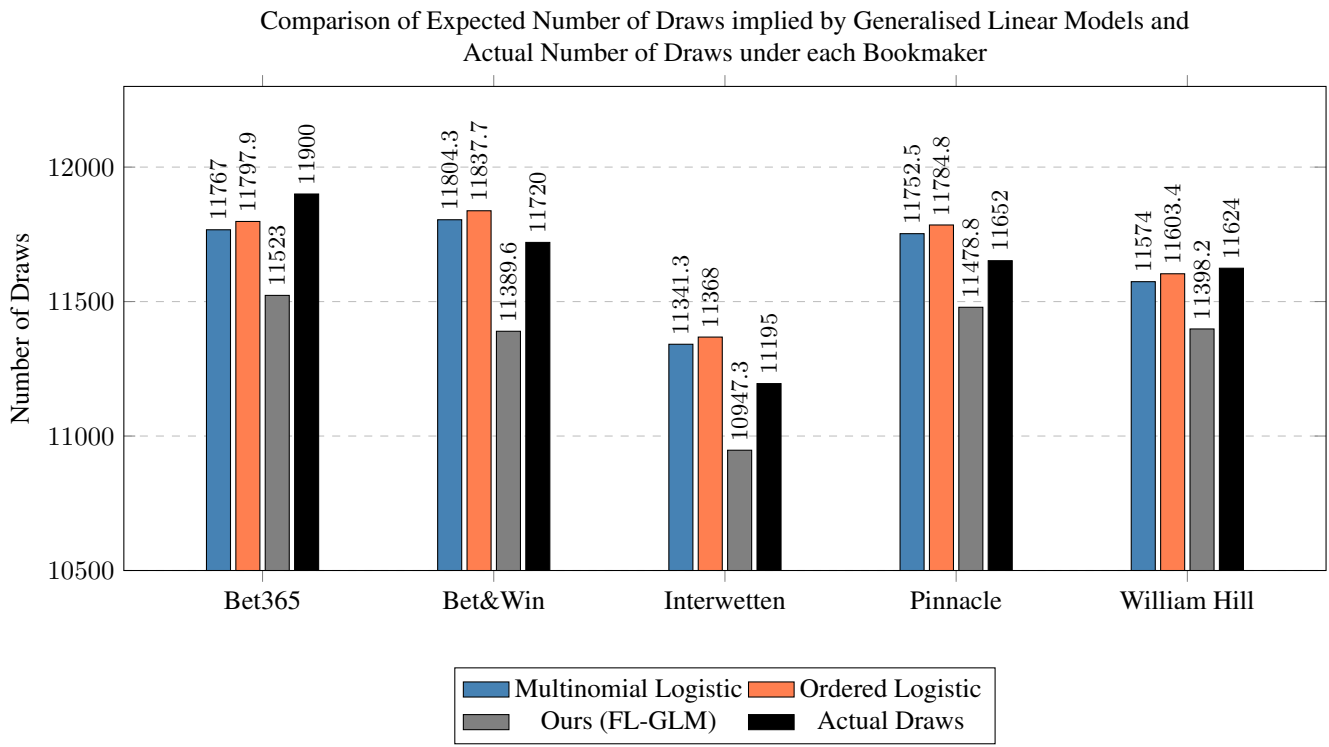
\begin{figure*}[htbp]
\centering
\begin{tikzpicture}
\begin{axis}[
    ybar,
    area legend,
    enlarge x limits=0.15,
    legend style={at={(0.5,-0.2)}, anchor=north, legend columns=2},
    ylabel={Number of Draws},
    symbolic x coords={Bet365, Bet\&Win, Interwetten, Pinnacle, William Hill},
    xtick=data,
    scaled y ticks=false,
    yticklabel style={/pgf/number format/fixed, /pgf/number format/1000 sep={}},
    nodes near coords,
    every node near coord/.append style={rotate=90, anchor=west, font=\small, text=black, /pgf/number format/1000 sep={}},
    ymin=10500, ymax=12300, 
    width=\textwidth,
    height=8cm,
    bar width=9pt, 
    title={Comparison of Expected Number of Draws implied by Generalised Linear Models and \\ Actual Number of Draws under each Bookmaker},
    title style={align=center},
    ymajorgrids=true,
    grid style=dashed,
]

\addplot[fill=myblue, draw=black!100] coordinates {(Bet365,11767.0) (Bet\&Win,11804.3) (Interwetten,11341.3) (Pinnacle,11752.5) (William Hill,11574.0)};

\addplot[fill=myorange, draw=black!100] coordinates {(Bet365,11797.9) (Bet\&Win,11837.7) (Interwetten,11368.0) (Pinnacle,11784.8) (William Hill,11603.4)};

\addplot[fill=black!50, draw=black!100] coordinates {(Bet365,11523.0) (Bet\&Win,11389.6) (Interwetten,10947.3) (Pinnacle,11478.8) (William Hill,11398.2)};

\addplot[fill=black!100] coordinates {(Bet365,11900) (Bet\&Win,11720) (Interwetten,11195) (Pinnacle,11652) (William Hill,11624)};

\legend{Multinomial Logistic, Ordered Logistic, Ours (FL-GLM), Actual Draws}
\end{axis}
\end{tikzpicture}
\caption{Comparison of expected \textbf{Number of Draws} implied by each \textbf{Generalised Linear Model} and the Actual Draws under each bookmaker. Each given group of four bars represent an evaluation under the given bookmaker; each given colour represents the amount of expected draws implied by the given generalised linear model or the actual number of draws.}
\label{fig:regression_draws_chart}
\end{figure*}

Our proposed FL-GLM tends to underestimate the probability of draw outcomes. Based on a two-tailed Poisson significance test at the 0.05 level, the expected number of draws implied by our proposed FL-GLM is significantly lower than the actual number of draws for Bet365, Bet\&Win, Interwetten and William Hill. In contrast, the expected number of draws implied by the multinomial and ordered logistic generalised linear models does not significantly differ from the actual number of draws observed for all five bookmakers.

We hypothesise that draw outcomes have a weaker favourite–longshot bias than decisive outcomes. And because our proposed favourite-longshot-bias-adjusted generalised linear model applies a uniform favourite-longshot bias adjustment across all outcomes, we underestimate draw probabilities because draw outcomes tend to be longshots. Thus, we could enhance our proposed FL-GLM by fitting a power constant for draw outcomes separately from decisive outcomes to explicitly adjust for the weaker favourite-longshot bias for draw outcomes.

\section{Discussions}
\label{sec:5}

In this section, we discuss the results of our experiments in Section \ref{subsec:experiments}, future work that can be done to significantly enhance our work in Section \ref{subsec:futureWork}, and results from public forecasting competitions to examine our odds-only method's applicability under real-world uncertainty in Section \ref{subsec:4.2}.

\subsection{Discussion for Experiments}
\label{subsec:experiments}

Our proposed OO-EPC method outperforms existing counterparts for most bookmakers, as shown in \autoref{tab:joined}. We hypothesise that this is because other odds-only methods make assumptions with no supportive evidence, whereas our proposed OO-EPC method's assumption aligns with bookmakers' goals, as proven by Proposition \ref{prop:xxx}.

Proposition \ref{prop:1} and Proposition \ref{prop:2} prove that the numerical variant and analytical variant of the Shin conversion assume that bookmakers with smaller expected losses for bettors have less accurate betting odds. However, in \autoref{tab:5}, we found no statistically significant correlation between average booksum and mean log-loss. Thus, we found no evidence supporting the assumptions made by the numerical and analytical variants of the Shin conversion.

Property \ref{prop:3} proved that the multiplicative conversion is equivalent to the proposed FL-GLM when the power constant is assumed to be 1, and Property \ref{prop:4} proved that the power conversion is equivalent to the proposed FL-GLM when the normaliser is assumed to be 1. However, on \autoref{tab:8}, we found that the optimal power constant for our proposed FL-GLM is greater than 1 for all bookmakers, and thus we have no evidence supporting the assumption made by multiplicative conversion. Also, on \autoref{tab:8}, we found that the mean normaliser is less than 1 for all bookmakers, and thus we have no evidence supporting the assumption made by the power conversion.

The favourite-longshot bias exists to some extent for all bookmakers because the optimal power constant for our proposed FL-GLM is greater than 1 for all bookmakers on \autoref{tab:8}. Larger optimal power constants imply a stronger favourite-longshot bias because the multiplicative conversion assumes that the favourite-longshot bias does not exist and is equivalent to the proposed FL-GLM when the optimal power constant is 1. Therefore, the optimal power constant values for our proposed FL-GLM can be used to compare the extent of the favourite-longshot bias between groups of betting odds. For example, we could measure the optimal power constant values for different time periods to observe how the favourite-longshot bias changes over time.

All odds-only methods that attempt to adjust for the favourite-longshot bias tend to underestimate the probability of draws, as shown on Figure~\ref{fig:draws_chart}. These methods assume the favourite-longshot bias exists to the same extent for all outcomes, including draws. Expected number of draws implied by the multiplicative conversion, which does not adjust for the bias, tends not to significantly deviate from the actual number of draws. This suggests that the favourite-longshot bias is weaker for draw outcomes than for decisive outcomes.

Our proposed FL-GLM outperforms both multinomial and ordered logistic regression, as shown on \autoref{tab:joined_bootstrap_regression}. We hypothesise that this is because both existing counterparts attempt to capture relationships that the betting odds already capture under EMH.

However, our proposed FL-GLM underestimates the probability of draws for all bookmakers, significantly for four out of the five, as shown in Figure~\ref{fig:regression_draws_chart}. On the other hand, the two existing logistic models do not show a significant draw bias. We hypothesise that this is because the logistic regression models effectively use separate parameters for each outcome and thus can fit betting odds for draws separately; whereas our proposed FL-GLM uses one global parameter $\beta$ for all outcomes, and cannot adjust for the weaker favourite-longshot bias for draw outcomes.

\subsection{Discussion for Future Work}
\label{subsec:futureWork}

Future work is needed to alleviate the draw bias that exists in all bias-adjusted odds-only methods shown on Figure~\ref{fig:draws_chart} and in our proposed FL-GLM shown on Figure~\ref{fig:regression_draws_chart}. We hypothesise that an extension of our proposed FL-GLM could address this issue. Adjustments for other biases in betting odds, such as differences in market participation \cite{SportsBettingMarketWisdom} and match attendance \cite{BettingOddsBiases}, also require future work.

One extension of our FL-GLM is to have two power constants. One for the draw outcome and the other for the decisive outcomes. We should return the maximum likelihood estimate of the power constants. If our hypothesis that the favourite-longshot bias is weaker in betting odds for draws than decisive outcomes is true, then the power constant for the draw outcome should be smaller than the power constant for the decisive outcome in the maximum likelihood estimate of power constants. The likelihood function can be defined as:
\newline
\newline
$X_{decisive}$ denotes column vectors of betting odds for decisive outcomes of each game
\newline
\newline
$X_{draw}$ denotes column vectors of betting odds for the draw outcome of each game
\newline
\newline
$\beta_{decisive}$ denotes the power constant parameter for decisive outcomes
\newline
\newline 
$\beta_{draw}$ denotes the power constant parameter for the draw outcome
\newline
\newline
We assume $X_{decisive}$ and $X_{draw}$ have same ordering of games
\[
\mathcal{D} = \left\{ X_{decisive}, X_{draw} \right\}
\]
\[
\hat{Y}_{decisive} = X_{decisive}^{\beta_{decisive}} 
\]
\[
\hat{Y}_{draw} = X_{draw}^{\beta_{draw}} 
\]
\[
\hat{Y} = 
\begin{bmatrix}
\hat{Y}_{decisive} & \hat{Y}_{draw}
\end{bmatrix}
\]
\[
\bar{\hat{Y}}_{ij} = \frac{\hat{Y}_{ij}}{\sum_{k=1}^{n} \hat{Y}_{ik}}
\quad 
\] 
\[\text{for } i = 1, 2, \dots, m,\; j = 1, 2, \dots, n
\]
\newline
\newline
$p(\mathcal{D} \mid \left\{\beta_{decisive}, \beta_{draw}\right\}) = \sum_{i=1}^{n} \sum_{j=1}^{k} Y_{ij} \ln \bar{\hat{Y}}_{ij}$ denotes the likelihood function.

\subsection{Informal Discussion for Predictive Modelling Competition Results}
\label{subsec:4.2}

Kaggle is an online platform that hosts predictive modelling competitions and allows participants to share open-source repositories with a global community. To evaluate our proposed OO-EPC method under real-world uncertainty, it has been available to the Kaggle community as an open-source repository \cite{GotoConversion} from 2024 onwards, and we used it in our solution since 2019 \cite{GotoConversionKaggle}.

Kaggle \cite{Kaggle} hosts an annual competition where participants submit the estimated probability of each game outcome for every possible match-up during both the men's and women's annual National Collegiate Athletic Association (NCAA) basketball tournaments. In 2024 \cite{Kaggle2024}, the format was slightly different, where participants essentially estimated the probabilities of each team progressing to each stage of the tournament. The submission deadline is before the tournaments begin to reflect real-world uncertainty.

Our proposed OO-EPC method has been publicly acknowledged by more than 10 gold medal-winning and more than 100 medal-winning solutions \cite{GotoConversion}. For six iterations, we used the proposed OO-EPC method in our solution, and the results are analysed in this section. Because we performed additional data processing to remain consistent with the competition format, we discuss how our proposed OO-EPC method is practically useful as a component of Kaggle competition solutions, but we cannot state causal relationships between using our proposed OO-EPC method and Kaggle competition results. Thus, the analysis is not a formal controlled experiment but rather an informal analysis of a solution that uses our proposed OO-EPC method under real-world uncertainty.

Section \ref{subsubsec:4.2.1} tests whether our solution provides a significant improvement in achieving top 10\% finishes. Section \ref{subsubsec:4.2.2} tests whether our solution provides a significant improvement in winning medals. Section \ref{subsubsec:4.2.3} compares the predictability of the men's and women's tournament game outcomes.

\subsubsection{Predictive Modelling Competition Results by Percentile}
\label{subsubsec:4.2.1}

Table~\ref{tab:19} provides the medals and percentiles achieved by our solution for Kaggle’s annual college basketball prediction competitions from 2019 to 2025. For each competition iteration, the table shows our percentile ranking on the leaderboard. We assume the average participant has a uniform probability of placing anywhere on the leaderboard.

Our solution achieved a top 10\% finish in 5 out of the 9 competitions listed. The probability of an average participant achieving 5 or more top 10\% finishes in 9 independent attempts is about $0.0009$ under a binomial distribution hypothesis test. This provides evidence that the frequency of top 10\% finishes achieved by our solution, which uses our proposed OO-EPC method, is statistically significant.

Thus, we have evidence that our solution results in top 10\% finishes significantly more frequently than the average participant, as shown in Table~\ref{tab:19}. 

\begin{table}[htbp]
    \centering
    \resizebox{0.7\columnwidth}{!}
    {%
    \begin{tabular}{|l|c|}
    \hline
    \textbf{Iteration} & \textbf{Result} \\
    \hline
    2019 Mens & No Medal \\
    2019 Womens & Silver Medal (95th percentile) \\
    2021 Mens & No Medal \\
    2021 Womens & Silver Medal (90th percentile) \\
    2022 Mens & No Medal \\
    2022 Womens & Gold Medal (98th percentile) \\
    2023 Combined & No Medal \\
    2024 Combined & Bronze Medal (91st percentile) \\
    2025 Combined & Bronze Medal (90th percentile) \\
    \hline
    \end{tabular}%
    }
    \caption{For each \textbf{Iteration} of Kaggle's annual college basketball prediction competition, the medal we won and our final leaderboard percentile are displayed in the \textbf{Result} column.}
    \label{tab:19}
\end{table}

\subsubsection{Predictive Modelling Competition Results by Medal Line}
\label{subsubsec:4.2.2}

Table~\ref{tab:20} provides the success rate for different criteria using our solution for Kaggle’s annual college basketball prediction competitions from 2019 to 2025. We defined our six success criteria to align with the medal thresholds, because the typical motivation of Kaggle participants is to win medals. We assume that an average participant has a uniform probability of placing anywhere on the leaderboard, and for each success criterion, the table shows the average participant achieving our success rate.
    
We performed two-tailed binomial distribution hypothesis tests at a 0.05 threshold. For the medal and silver medal criteria, we have significant evidence against the null hypothesis. For the gold medal criteria, we still have evidence against the null hypothesis, but it is insignificant.

Thus, we have significant evidence that our solution, which uses our proposed OO-EPC method, has a superior success rate for winning bronze and silver medals compared to the average participant. Although evidence for the success rate of winning gold medals is statistically insignificant, our solution still outperforms the success rate of the average participant for winning gold medals.

\begin{table}[htbp]
    \centering
    \resizebox{\columnwidth}{!}
    {%
    \begin{tabular}{|l|c|c|}
    \hline
    \textbf{Success Criteria} & \textbf{Success Rate} & \textbf{p-value for Average Participant} \\
    \hline
    Medals & 5/9 = 55.6\% & 0.0038 \\
    Silver Medals & 3/9 = 33.3\% & 0.0229 \\
    Gold Medals & 1/9 = 11.1\% & 0.1225 \\
    Medals for Women’s included Iterations & 5/6 = 83.3\% & 0.0005 \\
    Silver Medals for Women’s included Iterations & 3/6 = 50.0\% & 0.0080 \\
    Gold Medals for Women’s included Iterations & 1/6 = 16.7\% & 0.0867 \\
    \hline
    \end{tabular}%
    }
    \caption{For each \textbf{Success Criteria}, our \textbf{Success Rate} is displayed. The cumulative probability of the average participant achieving such a success rate or better for the given success criteria is displayed in the \textbf{p-value for Average Participant} column.}
    \label{tab:20}
\end{table}

\subsubsection{Men's and Women's Tournament Predictability Comparison}
\label{subsubsec:4.2.3}

Table~\ref{tab:21} compares the medal-winning log-loss thresholds from 2019 to 2022 for men's and women's tournaments. From 2023 onwards, Kaggle started hosting the annual competition with the predictions for men's and women's tournaments combined into one competition. The medal-winning log-loss thresholds were lower for the women’s tournament games than for the men’s, across all years for all types of medals. This is consistent with how women's tournament games have higher predictability than men's.

As shown in Tables~\ref{tab:19} and~\ref{tab:20}, our solution has led to stronger finishes and success rates in competitions that included women’s tournaments. We hypothesise that this is caused by less noise on the leaderboard when the women's tournament is included, resulting in less clear superior performance of our solution compared to the average participant.

\begin{table}[htbp]
    \centering
    \resizebox{\columnwidth}{!}{%
    \begin{tabular}{|l|c|c|c|}
    \hline
    \textbf{Iteration} & \textbf{Gold Medal} & \textbf{Silver Medal} & \textbf{Bronze Medal} \\
    \hline
    2019 Mens & 0.43815 & 0.45214 & 0.46116 \\
    2019 Womens & 0.34042 & 0.36147 & 0.37425 \\
    2021 Mens & 0.56942 & 0.58990 & 0.60065 \\
    2021 Womens & 0.40776 & 0.43979 & 0.45302 \\
    2022 Mens & 0.58480 & 0.60358 & 0.61141 \\
    2022 Womens & 0.41333 & 0.42675 & 0.43643 \\
    \hline
    \end{tabular}%
    }
    \caption{For each \textbf{Iteration} of Kaggle's annual college basketball prediction competition (2019–2022) before the competition design was changed, the log-loss score required to win \textbf{Gold Medal}, \textbf{Silver Medal}, and \textbf{Bronze Medal} is displayed for the separate men’s and women’s tournament competitions.}
    \label{tab:21}
    \end{table}

\section{Conclusion}
\label{sec:7}

We provide a novel foundation for sports forecasting by proposing two novel methods to accurately use betting odds as a benchmark for sports forecasting models under different data availability situations. 

We firstly proposed a novel odds-only method for converting betting odds to probabilities without using historical betting odds, which outperforms the accuracy of existing odds-only methods for the majority of bookmakers. We provide novel propositions to prove that our proposed OO-EPC method converts betting odds to probabilities by aligning with the objectives of bookmakers, and also how both variants of the existing Shin method make assumptions that are not supported by what we observe from historical betting odds.

We secondly proposed a novel generalised linear model for converting betting odds to probabilities using historical betting odds for fitting, which outperforms the accuracy of existing generalised linear models. We also use properties of our proposed FL-GLM to provide evidence that the existing odds-only methods multiplicative and power methods make assumptions that contradict with what we observe from historical betting odds. Furthermore, our proposed FL-GLM allows for an interpretable measurement of the favourite-longshot bias.

\section*{Declarations}

\begin{itemize}
\item Author contribution: K.G. performed the experiments, wrote the manuscript text, tables and figures. All authors reviewed the manuscript.
\item Funding: This research did not receive funding.
\item Conflict of interest/Competing interests: I declare that the authors have no conflict/competing interests as defined by Nature Portfolio, or other interests that might be perceived to influence the results and/or discussion reported in this paper.
\item Data availability: Datasets obtained from Football-Data.co.uk were analysed in this study, and are publicly available at the following URL: https://www.football-data.co.uk/.
Datasets obtained from Kaggle Competitions were also analysed in this study, and are publicly available at the following URL: https://www.kaggle.com/competitions
\item Dual publication: results/data/figures in this manuscript have not been published elsewhere, nor are they under consideration by another publisher.
\item Authorship: corresponding author has read the journal policies and submitted this manuscript in accordance with those policies.
\item Third party material: all of the figures, images and supplementary material are owned by the authors and/or no permissions are required.
\end{itemize}

\bibliographystyle{sageh}

\bibliography{academicLibraryNature}

@Article{ShinConversionStrumbeljVariant,
  author    = {{\v{S}}trumbelj, Erik},
  journal   = {International Journal of Forecasting},
  title     = {On Determining Probability Forecasts from Betting Odds},
  year      = {2014},
  number    = {4},
  pages     = {934--943},
  volume    = {30},
  publisher = {Elsevier},
}

@Article{PowerConversion,
  author    = {Clarke, Stephen and Kovalchik, Stephanie and Ingram, Martin},
  journal   = {American Journal of Sports Science},
  title     = {Adjusting Bookmaker’s Odds to Allow for Overround},
  year      = {2017},
  issn      = {2330-8540},
  number    = {6},
  volume    = {5},
  publisher = {Swinburne},
}

@Article{ShinConversionKizildemirVariant,
  author    = {Kizildemir, Melis and Akin, Ertugrul and Alkan, Altug},
  journal   = {Journal of Quantitative Analysis in Sports},
  title     = {A Family of Solutions Related to Shin’s Model for Probability Forecasts},
  year      = {2025},
  publisher = {De Gruyter},
}

@Article{LongTermFootballPrediction,
  author    = {Constantinou, Anthony and Fenton, Norman},
  journal   = {Knowledge-Based Systems},
  title     = {Towards Smart-Data: Improving Predictive Accuracy in Long-Term Football Team Performance},
  year      = {2017},
  pages     = {93--104},
  volume    = {124},
  publisher = {Elsevier},
}

@Article{BettingOddsRatingSystem,
  author    = {Wunderlich, Fabian and Memmert, Daniel},
  journal   = {PloS one},
  title     = {The Betting Odds Rating System: Using Soccer Forecasts to Forecast Soccer},
  year      = {2018},
  number    = {6},
  pages     = {e0198668},
  volume    = {13},
  publisher = {Public Library of Science San Francisco, CA USA},
}

@Article{BettingOddsBiases,
  author    = {{\v{S}}trumbelj, Erik},
  journal   = {Journal of Sports Economics},
  title     = {A Comment on the Bias of Probabilities Derived from Betting Odds and their Use in Measuring Outcome Uncertainty},
  year      = {2016},
  number    = {1},
  pages     = {12--26},
  volume    = {17},
  publisher = {SAGE Publications Sage CA: Los Angeles, CA},
}

@Article{RatingsBettingOddsComparison,
  author    = {Leitner, Christoph and Zeileis, Achim and Hornik, Kurt},
  journal   = {International Journal of Forecasting},
  title     = {Forecasting Sports Tournaments by Ratings of (Prob)abilities: A Comparison for the EURO 2008},
  year      = {2010},
  number    = {3},
  pages     = {471--481},
  volume    = {26},
  publisher = {Elsevier},
}

@Article{FootballBothTeamScorePrediction,
  author    = {da Costa, Igor Barbosa and Marinho, Leandro Balby and Pires, Carlos Eduardo Santos},
  journal   = {International Journal of Forecasting},
  title     = {Forecasting Football Results and Exploiting Betting Markets: The Case of “Both Teams to Score”},
  year      = {2022},
  number    = {3},
  pages     = {895--909},
  volume    = {38},
  publisher = {Elsevier},
}

@Article{BettingOddsEfficiencyComparison,
  author    = {Angelini, Giovanni and De Angelis, Luca},
  journal   = {International Journal of Forecasting},
  title     = {Efficiency of Online Football Betting Markets},
  year      = {2019},
  number    = {2},
  pages     = {712--721},
  volume    = {35},
  publisher = {Elsevier},
}

@Article{BettingStrategyNBA,
  author    = {Hub{\'a}{\v{c}}ek, Ond{\v{r}}ej and {\v{S}}ourek, Gustav and {\v{Z}}elezn{\`y}, Filip},
  journal   = {International Journal of Forecasting},
  title     = {Exploiting Sports-Betting Market using Machine Learning},
  year      = {2019},
  number    = {2},
  pages     = {783--796},
  volume    = {35},
  publisher = {Elsevier},
}

@Article{BradleyTerryTennisPrediction,
  author    = {McHale, Ian and Morton, Alex},
  journal   = {International Journal of Forecasting},
  title     = {A Bradley-Terry Type Model for Forecasting Tennis Match Results},
  year      = {2011},
  number    = {2},
  pages     = {619--630},
  volume    = {27},
  publisher = {Elsevier},
}

@Article{WhyLogLoss,
  author    = {Wunderlich, Fabian and Memmert, Daniel},
  journal   = {International Journal of Forecasting},
  title     = {Are Betting Returns a Useful Measure of Accuracy in (Sports) Forecasting?},
  year      = {2020},
  number    = {2},
  pages     = {713--722},
  volume    = {36},
  publisher = {Elsevier},
}

@Article{SportsBettingMarketWisdom,
  author    = {Brown, Alasdair and Yang, Fuyu},
  journal   = {International Journal of Forecasting},
  title     = {The Wisdom of Large and Small Crowds: Evidence from Repeated Natural Experiments in Sports Betting},
  year      = {2019},
  number    = {1},
  pages     = {288--296},
  volume    = {35},
  publisher = {Elsevier},
}

@Article{SportsIndustryOverview,
  author  = {Bayarslan, Bahad{\i}r},
  journal = {Innovat{\i}ve Research In Sport Sc{\i}ences},
  title   = {Sports Industry Overview},
  year    = {2023},
  pages   = {5--20},
}

@InCollection{EfficientMarketHypothesis,
  author    = {Malkiel, Burton G},
  booktitle = {Finance},
  publisher = {Springer},
  title     = {Efficient market hypothesis},
  year      = {1989},
  pages     = {127--134},
  editor    = {Eatwell, John and Milgate, Murray and Newman, Peter}
}

@Article{FLBias,
  author    = {Griffith, Richard M},
  journal   = {The American Journal of Psychology},
  title     = {Odds Adjustments by American Horse-Race Bettors},
  year      = {1949},
  number    = {2},
  pages     = {290--294},
  volume    = {62},
  publisher = {JSTOR},
}

@Article{EloRating,
  author  = {Elo, Arpad E},
  journal = {Chess life},
  title   = {The Proposed USCF Rating System, its Development, Theory, and Applications},
  year    = {1967},
  number  = {8},
  pages   = {242--247},
  volume  = {22},
}

@Book{FLBiasStockMarkets,
  author    = {Hodges, Stewart D and Tompkins, Robert G and Ziemba, William T},
  publisher = {SSRN},
  title     = {The favorite/long-shot bias in S\&P 500 and FTSE 100 index futures options: the return to bets and the cost of insurance},
  year      = {2003},
}

@Software{GotoConversion,
  author = {Goto, Kaito},
  month  = aug,
  title  = {{Gambling Odds To Outcome probabilities Conversion (goto\_conversion)}},
  url    = {https://github.com/gotoConversion/goto_conversion},
  year   = {2023},
}

@InProceedings{MultiplicativeConversion,
  author       = {Chernov, Alexey and Vovk, Vladimir},
  booktitle    = {International Conference on Algorithmic Learning Theory},
  title        = {Prediction with expert evaluators’ advice},
  year         = {2009},
  publisher = {Springer},
  pages        = {8--22},
  editor    = {Gavald{\`a}, Ricard and Lugosi, G{\'a}bor and Zeugmann, Thomas and Zilles, Sandra},
}

@Article{FootballData,
  author = {Joseph Buchdahl},
  title  = {Football-Data.co.uk},
  year   = {2024},
  note   = {FootballData},
  url    = {https://football-data.co.uk/},
}

@Article{MultinomialLogisticRegression,
  author    = {Baxter, Mike},
  journal   = {The Mathematical Gazette},
  title     = {Generalised linear models , by P. McCullagh and JA Nelder. Pp 511.{\pounds} 30. 1989. ISBN 0-412-31760-5 (Chapman and Hall)},
  year      = {1990},
  number    = {469},
  pages     = {320--321},
  volume    = {74},
  publisher = {Cambridge University Press},
}

@Article{OrderedLogisticRegression,
  author    = {McCullagh, Peter},
  journal   = {Journal of the Royal Statistical Society: Series B (Methodological)},
  title     = {Regression models for ordinal data},
  year      = {1980},
  number    = {2},
  pages     = {109--127},
  volume    = {42},
  publisher = {Wiley Online Library},
}

@Article{PowerLawRegression,
  author    = {Clauset, Aaron and Shalizi, Cosma Rohilla and Newman, Mark EJ},
  journal   = {SIAM review},
  title     = {Power-law distributions in empirical data},
  year      = {2009},
  number    = {4},
  pages     = {661--703},
  volume    = {51},
  publisher = {SIAM},
}

@Misc{Kaggle2024,
  author       = {Jeff Sonas and Ryan Holbrook and Addison Howard and Anju Kandru},
  howpublished = {\url{https://kaggle.com/competitions/march-machine-learning-mania-2024}},
  note         = {Kaggle},
  title        = {March Machine Learning Mania 2024},
  year         = {2024},
}

@Misc{Kaggle,
  author       = {Addison Howard and Danielle Notaro and Eric Schmidt and Jeff Sonas and Jen Raulli and Rachel Ahn and Tiffany Martin and Will Cukierski},
  howpublished = {\url{https://kaggle.com/competitions/womens-machine-learning-competition-2019}},
  note         = {Kaggle},
  title        = {Google Cloud \& NCAA® ML Competition 2019-Women's},
  year         = {2019},
}

@Misc{GotoConversionKaggle,
  author       = {Goto, Kaito},
  howpublished = {\url{https://www.kaggle.com/kaito510}},
  title        = {Applying goto\_conversion to Kaggle's Annual Basketball Outcome Prediction Competition},
  year         = {2025},
}

@Book{sumpter2016soccermatics,
  author    = {David Sumpter},
  publisher = {Bloomsbury Sigma},
  title     = {Soccermatics: Mathematical Adventures in the Beautiful Game},
  year      = {2016},
  isbn      = {9781472924148},
  pages     = {351},
}

\end{document}